\documentclass[runningheads]{llncs}
\usepackage{graphicx}
\usepackage{amssymb, amsmath}
\usepackage{algorithmic, algorithm}
\usepackage{caption}
\usepackage{subcaption}
\usepackage{tabularx}
\usepackage{fancybox}

\newcommand{\R}{\mathbb{R}}
\newcommand{\pr}{\mathbb{P}}

\newcommand{\cD}{\mathcal{D}}
\newcommand{\cO}{\mathcal{O}}
\newcommand{\cN}{\mathcal{N}}
\newcommand{\cZ}{\mathcal{Z}}
\newcommand{\cT}{\mathcal{T}}
\newcommand{\cX}{\mathcal{X}}
\newcommand{\cW}{\mathcal{W}}

\newcommand{\la}{\lambda}
\newcommand{\tla}{\tilde{\lambda}}

\newcommand{\bv}{\mathbf{v}}
\newcommand{\tbv}{\tilde{\mathbf{v}}}

\newcommand{\bu}{\mathbf{u}}
\newcommand{\bw}{\mathbf{w}}
\newcommand{\bx}{\mathbf{x}}

\newtheorem{cor}[theorem]{Corollary}

\newtheorem{thm}{Theorem}[section]

\newtheorem{defn}{Definition}[section]

\begin{document}

\title{Multi-Slice Clustering for $3$-order Tensor Data}
\author{Dina Faneva Andriantsiory, Joseph Ben Geloun, Mustapha Lebbah}
%\institute{  Laboratoire d'Informatique de Paris Nord (LIPN) \\   Université Sorbonne Paris Nord}
\institute{
	Laboratoire d'Informatique de Paris Nord (LIPN) \\
	Université Sorbonne Paris Nord}

\maketitle

\begin{abstract}
	Several methods of triclustering of three dimensional data require the specification of the cluster size in each dimension. This  introduces a certain degree of arbitrariness. To address this issue, we propose
	a new method, namely the multi-slice clustering (MSC) for a $3$-order tensor data set.
	We analyse, in each dimension or tensor mode, the spectral decomposition of each tensor slice, i.e. a matrix. Thus, we define a similarity measure between matrix slices up to a threshold (precision) parameter, and from that, identify a cluster. 
	The intersection of all partial clusters provides the desired triclustering. The effectiveness of our algorithm is shown on both synthetic and real-world data sets. 
\end{abstract}

\section{Introduction}

Consider $m_1$ individuals with $m_2$ features and 
collect the data for each individual-feature pair at $m_3$ different times.  This is an example of a collection of data set with $3$ dimensions. Multidimensional data of this kind arise in sundry contexts such as neuroscience \cite{dataneuroscience}, genomics data \cite{datagenomics1,datagenomics2},  computer vision \cite{datacomputerVision} and several other domains \cite{HRMSC2018}.  
A convenient way to encode such data is given by  a $3$-order tensor. 
Because of their augmenting complexity, the mining of higher dimensional data naturally proceeds with the identification of subspaces with specific features. Clustering is one of the most popular unsupervised machine learning methods for extracting relevant information such as structure similarity in data.
Thus, several computational methods for clustering multidimensional data, from matrices  to higher order tensors, were introduced, see for example the review \cite{HRMSC2018}.
For a detailed account of various available methods closer to our present work, see 
\cite{MichaelKathrin2008NTF,Kolda2009,FeiziNIPS2017,DinaTB,TCarson2017,BachLearningSpectralClustering}.

A great deal of clustering techniques applied to tensors are based on 
the formalism of Kolda and
Bader \cite{Kolda2009}.
Following that trend,  
Feizi et al. in \cite{FeiziNIPS2017} proposed four algorithms to select the highly correlated  trajectories over one dimension (figure \ref{fig:tfs}). An extension of that work to select multiple clusters of trajectories is proposed by Andriantsiory et al. in  \cite{DinaTB}.
Later, Wang and collaborators \cite{WangMultiwayClustering2019}
set up partitions in each mode of the tensor and considered the intersection of three clusters from different modes as a block represented by a mean
(Multiway clustering). 
Another method is
the so-called CP+k-means, see for instance Papalexakis et al. \cite{Papalexakis2013}, 
and Sun and Li \cite{WillSun2019}. One starts with a CP-decomposition of the tensor into $R$ rank-1 tensors, and then apply the k-means clustering \cite{Wilkin2007} to the rows of each matrix factor. Similarly, the Tucker+k-means \cite{SunTuckerk_means2006}
uses a Tucker decomposition \cite{tucker1966} of the tensor and then apply a usual k-means clustering. 
We will be interested in the improvements of these algorithms in the present work. 
(Note that the Multiway clustering of Wang et al \cite{WangMultiwayClustering2019} 
will approximately behave
like the Tucker+k-means in our setting.)

In the usual sense, the triclustering tensor problem aims at computing a triple of subsets of features, individuals and times where the entries of the tensor are strongly similar and have a structure of block (a sub-cube) with 
higher means (figure \ref{fig:triclustering}). 
It is considered as a natural generalization
and next step of the tensor biclustering method.

In general, the clustering algorithms depend on the number of clusters or the cluster size as data input. Nevertheless, for real data, these inputs
might be very difficult to assess from the outset. 
Another important open issue about 
the inputs (cluster sizes) concerns 
the comparison of different 
clustering results from different parameter inputs. This leads to difficulties
in the actual choice of their  value. 
We must mention that the usual clustering
algorithms only guarantee a clear distinction between a given cluster 
and the rest of the data, without guaranteeing the strong correlation
within a cluster. One might ask both
properties: separation and strong
correlation within the cluster.

In this paper, we first bring an answer to the issue of the parameter sizes.   Our method
concentrates on one dimension of the tensor and studies each matrix slice of that mode. 
We select the top-eigenvectors of all slices and study a covariance-type matrix associated with those. 
This guides us towards the selection of an index subset  
with highly similar features.  
The input sizes of the cluster are not needed as input of the algorithm. Instead, 
we set a threshold parameter which gauges the similarity within the cluster. Our method leads to a well-defined triclustering with 
performance comparable to that of 
Tensor Folding Spectral (TFS) method \cite{FeiziNIPS2017} and of the Tucker+k-means (for $k=2$). 
Finally, we show experimentally, using a similarity index measure  within the cluster, that 
all vectors selected are always strongly 
correlated, a feature not necessarily shared by the other algorithms. 
%This demonstrates the coherence of our algorithm.

The paper is organized as follows: 
in section \ref{pbmformulation}, we address the problem formulation. The description of our method and algorithm is set in section \ref{method}. Then, in section \ref{experiments}, 
we apply our algorithm on synthetic data to compare the coherence of the theory and 
compare its performance with those of three known clustering algorithms. We also apply the algorithm on a real dataset and extract a tri-cluster. Section \ref{ccl} summarizes our results. The paper has a companion  supplementary material  that gathers the proofs of our main theorems
that is put at the end 
of this manuscript. 

	\begin{figure}[h]
	\centering
	\begin{subfigure}[b]{0.4\textwidth}
		\centering
		\includegraphics[width=\textwidth]{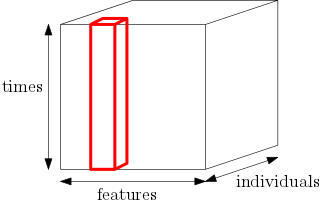}
		\caption{Tensor biclustering}
		\label{fig:tfs}
	\end{subfigure}
	\begin{subfigure}[b]{0.4\textwidth}
		\centering
		\includegraphics[width=\textwidth]{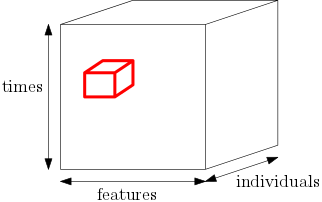}
		\caption{Triclustering}
		\label{fig:triclustering}
	\end{subfigure}
	\caption{(a) The tensor biclustering model and (b) The tensor triclustering model.}
	\label{fig:cluster}
\end{figure}

\section{Problem formulation}\label{pbmformulation}

\noindent{\bf Notation --}
We use $\cT$ to represent the 3-order tensor data and use Matlab notation in next manipulations. For the matrices, we use  capital letters: $M, C,\cdots$, and $\|M\|$ is  the operator norm of $M$. The lowercase boldface letters $\bx, \bv, \cdots$ represent vectors and $\|\bx\|_2$ is the euclidean norm of the vector $\bx$. For any set $J$, $|J|$ denotes the cardinality of $J$ and $\bar{J}$ denotes the complementary of $J$ in a larger set. For an integer $n>0$, we denote $[n] = \{1,\cdots,n\}$.  The asymptotic notation $a(n) = \cO(b(n))$ (res. $a(n) = \Omega(b(n))$) means that, there exists a universal constant $c$ such that for sufficiently large $n$, we have $|a(n)| \le c b(n)$ (resp. $|a(n)| \ge  c b(n)$).

Let $\cT = \cX + \cZ$ with $\cT\in\R^{m_1\times m_2\times m_3}$, where $\cX$ is the signal tensor and $\cZ$ is the noise tensor. Consider

\begin{equation}
	\cT = \cX + \cZ = \sum_{i = 1}^{r}\gamma_i\, \bw^{(i)}\otimes \bu^{(i)}\otimes\bv^{(i)} + \cZ \label{eq:problem}
\end{equation}
where $\forall i, \gamma_i>0$ stands for the signal strength,  $\bw^{(i)}\in \R^{m_1}, \bu^{(i)}\in\R^{m_2}$ and $\bv^{(i)}\in\R^{m_3}$ are normed vectors, and $r$ is a number of rank-1 tensor (vector) decomposition of the signal tensor. In this work, we restrict to the case $r = 1$.
\begin{defn}[Triclustering]
	The problem of tensor triclustering aims  at computing a triple $(J_1, J_2, J_3)$ of index sets  with $J_i\subset[m_i],i=1,2,3$, where the entries of $\cT$ are 
	highly correlated (similar).
\end{defn}
To simplify notation, we drop the superscripts $(i)$ from $\bv^{(i)}$ and from the other vectors.  For every $(j_1,j_2,j_3)\in J_1\times J_2\times J_3, |\bw(j_1)|\le \delta, |\bu(j_2)|\le \delta$ and $| \bv(j_3) |\le \delta$ for a constant $\delta >0$ and $\cX_{j_1,j_2,j_3}=0$ for each $(j_1,j_2,j_3)$ outside $J_1\times J_2\times J_3$. 
Concerning the noise mode, for $(j_1,j_2,j_3) \in \bar{J}_1\times \bar{J}_2\times\bar{J}_3 $, we assume that the entries of $\cZ$ are i.i.d and have a standard normal distribution. 
This is the conventional noise model in unsupervised learning method for tensor data \cite{FeiziNIPS2017,Cai_2017}.

\section{Method}
\label{method}

\noindent{\bf Multi-slice clustering --}
Consider the triclustering problem \eqref{eq:problem}. Using 
the notation of Kolda and Bader \cite{Kolda2009}, let
\begin{equation}
	\label{tmatrix}
	T_i = \cT(i,:,:)=X_i + Z_i,\qquad i\in[m_1]
\end{equation}
be the horizontal (the first mode) matrix slice of the tensor $\cT$, with $X_i = \cX(i,:,:)$ and $Z_i = \cZ(i,:,:)$. Similarly, we can define the lateral (the second mode)  and the frontal (the third mode) matrix slices. A way to learn the embedded triclustering is to find the similarity of the slices in each mode. This yields an index subset  $J_i$ for each mode. Now, the collection of $J_i$ determines the tricluster.
\begin{defn}[Multi-slice clustering (MSC)] 
	The problem of multi-slice clustering of a $3$-order tensor $\cT$ aims at determining and gathering the indices of the matrix slices that are similar in each dimension of the tensor. 
\end{defn}
To illustrate the procedure, we will focus on the horizontal slices without loss of generality. 

We perform a spectral analysis of
\begin{equation}
	C_i = T_i^tT_i \,,  \quad i\in [m_1]. 
	\label{covariance}
\end{equation}
For each $i\in[m_1]$, $C_i$ represents the symmetric covariance matrix of the $i$-th horizontal slice. 

For $i\in J_1$, equation \eqref{eq:problem} shows that
\begin{equation}
	\label{cXX}
	C_i = (X_i + Z_i )^t(X_i + Z_i ) = X_i^tX_i + X_i^tZ_i + Z_i^tX_i + Z_i^tZ_i 
\end{equation}
Taking into consideration \eqref{eq:problem}
and \eqref{tmatrix}, a tensor calculation 
shows that the matrix 
$ X_i^tX_i  = \la_i \bv_i\bv_i^t$, 
where $\la_i = \gamma_i^2 (\bw(i))^2>0$ is the top eigenvalue
of  $ X_i^tX_i$ that corresponds to 
the top eigenvector $\bv_i$. We re-express 
$ C_i = \la_i \bv_i\bv_i^t + W_i $ 
with $W_i$ the remainder 
inferred from the equation \eqref{cXX}.
There is a perturbation 
bound that shows the norm
difference between 
the top eigenvectors of
$C_i$ and $\bv_i$, see
supplementary material, section \ref{app:prooflemmaCvv}.  

For $i\notin J_1$, 
the expression $C_i$ \eqref{covariance}  becomes the covariance of the noise slice:
$C_i = Z_i^t Z_i$
where each row of $Z_i$ is real and independently drawn form $\cN_{m_3}(0, I)$, the $m_3$-variate normal distribution with zero mean and the covariance matrix is the identity matrix $I$. The matrix $C_i$ has a white Wishart distribution $\cW_{m_3}(m_2, I)$ \cite{johnstone2001}. The largest eigenvalue distribution $\la_i$ of $C_i$ when $m_2\rightarrow \infty, m_3\rightarrow \infty$ and $\frac{m_2}{m_3}\rightarrow {\rm cst} >1$, obeys the limit in the distribution sense
\begin{equation}\label{convergenceDistribution}
	\frac{\la_i - \mu_{m_2m_3,1}}{\sigma_{m_2m_3,1}}\xrightarrow[]{\cD} F_1
\end{equation}
where $F_1$ is the Tracy-Widom (TW) cumulative distribution function (cdf), and 
\begin{eqnarray}
	\mu_{m_2m_3,1} &=& (\sqrt{m_2-1} + \sqrt{m_3})^2\crcr
	\sigma_{m_2m_3,1} &=& \sqrt{\mu_{m_2m_3,1}} \big(\frac{1}{\sqrt{m_2 - 1}}+\frac{1}{\sqrt{m_3}}\big)^{\frac{1}{3}}
\end{eqnarray}
As stated in \cite{johnstone2001},  the expression \eqref{convergenceDistribution} applies equally well if $m_2<m_3$ when $m_2\rightarrow \infty, m_3\rightarrow \infty$ and if the role of $m_2$ and $m_3$ are reversed in the expression of $\mu_{m_2m_3,1}$ and $\sigma_{m_2m_3,1}$. Although \eqref{convergenceDistribution} is true in the limit, it is also shown that it provides a satisfactory approximation from $m_2, m_3 \ge 10$. Our following computations lie above that approximation. 
In the next developments, we drop the subscript $(m_2m_3,1)$ in $\mu_{m_2m_3,1}$ and $\sigma_{m_2m_3,1}$, so we will only use $\mu$ and $\sigma$. 

The covariance matrix $C_i$ defines both the spread (variance) and the orientation (covariance) of our
slice data. We are interested in the vector that represents the covariance matrix and its magnitude. The idea is to find the vector that points into the direction of the largest spread of the data $(\la_i, \bv_i)$.

We now detail our strategy. 
Since the signal tensor has rank one, we undertake the analysis of the relation between two slices. 
We seek for criteria characterizing 
high similarity between 
the top eigenvectors and
top eigvenvalues of all slices. 
A first necessary condition 
for achieving a MSC is to target 
high correlation 
between the top eigenvectors
of the slices. 
To proceed, 
we build the covariance matrix
$C= (c_{ij})$  of the top eigenvectors  with signs removed. 
We claim that two slices 
$i$ and $j$  are similar if the
corresponding coefficient
$c_{ij}$ is large enough.
At this stage, we set up 
our precision parameter $\epsilon$
that determines the correlation
strength between slices and, at a threshold, 
call these slices $\epsilon$-similar. We introduce a vector $\mathbf{d}$
the coordinate $d_i$ of which 
is the marginal sum of $C$
associated with the slice $C_i$. 
Any slice $i$ that belongs to the cluster
can be regarded as a centroid. 
Any other slice $j$ in the same
cluster is highly similar 
to $i$ 
and therefore will contribute to $d_i$.
An increasing number
of such $j$ makes
$d_i$ larger. 
So, for $i$ in the cluster and $j$ outside the cluster, we have $d_i>d_j$ with high probability.
The next step amounts
to set up a notion of distance between 
the cluster and the rest. 
We consider this as a criterion to determine the initial cluster.  Then, we use the parameter $\epsilon$  to fit the right cluster.  

For each slice, the  relation  \eqref{covariance} holds for $i= 1,\cdots, m_1$.
Let $\underline\la_i$ be
the top eigenvalue
and $\tbv_i$ be the 
top eigenvector of $C_i$.
We construct the matrix 
\begin{equation}
	V = \begin{bmatrix}
		\tla_1\tbv_1 & \cdots & \tla_{m_1}\tbv_{m_1}	\end{bmatrix}
\end{equation}
where  we set $\tla_i$ to $ \underline\la_i/\la, \forall i\in[m_1]$, and $\la =  \max(\underline\la_1,\cdots,\underline\la_{m_1})$.
Let $C$ be the matrix with each entry $(c_{ij})_{i,j\in[m_1]}$ defined by 
\begin{eqnarray}
	c_{ij}  = | \langle \tla_i \tbv_i,\tla_j \tbv_j \rangle| = \tla_i\tla_j| \langle  \tbv_i, \tbv_j \rangle| \le 1
	\label{scalar_product}
\end{eqnarray}
$C$ is a symmetric matrix and its  columns are not unit vectors. Note also that $C$ is related to the covariance matrix $V^tV$ in which we take the absolute values of the entries.

\begin{defn}[$\epsilon$-similarity]  The slice $i$ and $j$ are called $\epsilon$-similar, if for a small $\epsilon>0$ we have,
	\begin{equation}
		c_{ij} \geq 1 - \epsilon/2 \label{select}
	\end{equation}
\end{defn}
Let $\mathbf{d} = (d_1, \cdots, d_{m_1})$ a real vector defined by the marginal sum in $C$: $\forall i \in [m_1] $
\begin{equation}
	d_i =  \sum_{j\in [m_1]} c_{ij} \label{26}
\end{equation}
Using \eqref{select} and \eqref{26}, the elements of $J_1$ are indices $i\in [m_1]$ which define  pairwise $\epsilon$-similar slices and  thus obeys 
\begin{equation}\label{eq:d>k(1-e)}
	d_i \ge l (1- \epsilon/2)
\end{equation}
where $l = |J_1|$. Note that in the previous equation $l$ is unknown
(does not define a final criterion to 
select the indices of the cluster). However, it
already shows that $d_i$ is bounded from below and non negligible as soon as $l$ is large enough. The algorithm aims at computing $l$, in such way that only  $\epsilon$-similar slices
that satisfy \eqref{eq:d>k(1-e)} define
the cluster $J_1$. 
The following statements demonstrate that this
event happens with overwhelming probability. %%

\begin{thm}\label{thm1}	Let $l = |J_1|$, assume that $\sqrt{\epsilon}\le \frac{1}{m_1-l}$. $\forall i,n\in J_1 $, for $\la = \Omega(\mu)$, there is a constant $c_1>0$ such that
	\begin{equation}
		\label{didn}
		|d_i - d_{n}| \le l\frac{\epsilon}{2} + \sqrt{\log(m_1-l)}
	\end{equation}
	holds with probability at least $1-e(m_1-l)^{-c_1}$.
\end{thm}
\begin{proof} See supplementary material, section \ref{app:proofthm1}. 
\end{proof}
We inspect closer the previous inequality, and use $\sqrt{\epsilon}\le 1/(m_1-l)$
to write 
$|d_i - d_{n}| \le  l\frac{\epsilon}{2} + 
\sqrt{ \frac{1}{2} |\log\epsilon|  }$. 
Thus,  as the logarithm may indeed increase,
to choose a smaller $\epsilon$ might not lead to a smaller radius of the cluster. There is therefore an unknown threshold for $\epsilon$ that will lead to
the best clustering. Another remark is the following: 
since we must determine $l$ by iteration and that value 
turns out to decrease, the relation \eqref{didn} holds with a higher probability at each step of the iteration. 

\begin{thm}\label{thm2}
	For $i\in\bar{J}_1$, if $\la = \Omega(\mu m_1)$,
	\begin{equation}
		d_i\le \frac{l}{m_1} + \sqrt{\log(m_1-l)}
	\end{equation} 
	with probability at least $1-e(m_1-l)^{-c_1}$ with $c_1>0$. 
\end{thm}
\begin{proof} 
	See supplementary material, section \ref{app:proofthm1}. 
\end{proof}
The distance ${\rm dist}(A, B)$ between two sets of real numbers $A$ and $B$ is naturally defined as 
$ \min\{|x-y| : x\in A, y\in B\}$. 
We introduce 
$d(J_1)=\{d_i: i \in J_1\}$, 
and $d(\bar{J}_1)=\{d_i: i \in \bar{J}_1\}$. 

\begin{cor}\label{corollary}
	Let $J_1\subset [m_1]$ the set of all indices of the cluster in the first dimension. Then
	\begin{equation}
		{\rm dist}(d(J_1), d(\bar{J}_1))\ge l(1-\frac{\epsilon}{2} - \frac{1}{m_1}) - \sqrt{\log(m_1 - l)}  , \end{equation}
	with probability at least $1-e(m_1-l)^{-c_1}$, with $c_1>0$. 
\end{cor}
\begin{proof} This is an easy calculation starting from \eqref{eq:d>k(1-e)} and 
	using theorem \ref{thm2}. \end{proof}

The corollary \ref{corollary} proves that the distance ${\rm dist}(d(J_1), d(\bar{J}_1))$ increases when the number of elements that belong to the cluster increases. 

\begin{figure}[t]
	\centering
	\includegraphics[width=5cm, height=3.5cm ]{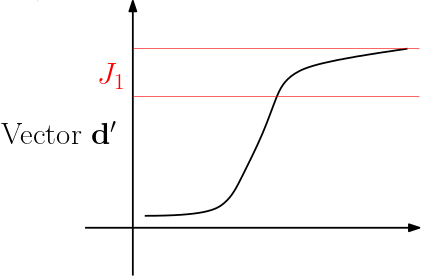}
	\caption{Representation of the vector $\mathbf{d}$ after reordering its elements in ascending order.}
	\label{fig:vectord1}
\end{figure}

\begin{algorithm}
	\caption{Multi-slice clustering for 3-order tensor}
	\begin{algorithmic}[1]
		\label{algo:multislice}
		\renewcommand{\algorithmicrequire}{\textbf{Input:}}
		\renewcommand{\algorithmicensure}{\textbf{Output:}}
		\REQUIRE $3$-order tensor $\cT\in\R^{m_1\times m_2\times m_3}$, real parameter $\epsilon>0$
		\ENSURE  sets $J_1, J_2$ and $J_3$
		\FOR {$j$ in $\{1,2,3\}$}
		\STATE Initialize the matrix $M$
		\STATE 	Initialize $ \la_0\gets 0$
		\FOR{ i in $\{1,2,\cdots,m_j\}$}
		\STATE Compute : $C_i\gets T_i^t T_i $
		\STATE Compute the top eigenvalue and eigenvector $(\la_i, \bv_i)$ of $C_i$
		\STATE Compute $M(:,i) \gets \la_i * \bv_i$
		\IF {$\la_i >\la_0$}
		\STATE $\la_0 \gets \la_i$
		\ENDIF
		\ENDFOR
		\STATE Compute : $V \gets M / \lambda_0$
		\STATE Compute : $C \gets |V^t V|$
		\STATE Compute : $\mathbf{d}$ the vector marginal sum of $C$ and sort it
		\STATE Initialization of $J_j$ using the maximum gap in $\mathbf{d}$ sorted
		\STATE Compute : $l\gets |J_j|$
		\WHILE {not convergence of the elements of $J_j$ (theorem \ref{thm1})}
		\STATE Update the element of $J_j$ (excluding $i$ s.t. $d_i$ is the smallest value that violates theorem \ref{thm1})
		\STATE 	Compute $l$
		\ENDWHILE
		\ENDFOR
		% \RETURN $(J_1, J_2, J_3)$ 
	\end{algorithmic} 
\end{algorithm}

Consider, $\forall i \in [m_3], d_i'\in \{d_1,\cdots,d_{m_3}\}$, the sequence of $d_i$ ordered in increasing values, $d_1'\le d_2'\le\cdots\le d_{m_3}'$ and  $\mathbf{d'} = (d_1', \cdots, d_{m_3}')$.
The figure \ref{fig:vectord1} represents the global view of the vector $\mathbf{d}'$ structure. $J_1$ defines
the highest indices of
the vector $\mathbf{d}'$.
Figure \ref{fig:vectord1} exhibits  
the fact corollary \ref{corollary} 
selects the indices of $J_1$ from the remaining indices with lower value.
Our algorithm uses the differences between two consecutive values of $d'_i$ to detect the gap whose corollary \ref{corollary} ensures the existence with high probability.

\noindent{\bf Computational complexity --}
A few words on the complexity of our	
algorithm is in order.  
To simplify the evaluation, we
require $m_i \in \Theta(n)$	for $i=1,2,3$. The algorithm \ref{algo:multislice} has a complexity class $\cO(n^3)$ that 
is comparable with the complexity of other triclustering algorithms. 
One may be puzzled about the cost of the last repeat loop depending on a convergence parameter. However, there is no issue with that since $l$ decreases at each turn. The convergence condition evaluates at most as $\cO(l^2)$,
and so this loop costs at most $\cO(l^3) \subset \cO(n^3)$.

\section{Experiments}
\label{experiments}

To evaluate the quality of the clustering yield by the algorithm, we first use two indices suitable for synthetic data, namely, the recovery rate, see \cite{FeiziNIPS2017}, and the Adjusted Random Index (ARI) in \cite{ARI}. Then, we also discuss the quality of the cluster with respect to the data similarity. 

\subsection{Synthetical results}
Here, we evaluate the performance of the MSC algorithm on synthetic datasets. To generate the input tensor $\cT$, we fix the three sets $J_1, J_2$,  and  $J_3$, and the entries of the vectors $\bv, \bu$ and $\bw$ in \eqref{eq:problem} equal to $1/\sqrt{|J_1|}, 1/\sqrt{|J_2|}$ and $1/\sqrt{|J_3|}$, respectively, in the cluster and zero outside the cluster. The noise is chosen standard normal distributed. 
The quality of the cluster result is evaluated by the recovery rate in each dimension and then we perform a mean of these index measures. 
Furthermore, we use a similarity
index (sim) that measures the quality of the correlation between the vectors of the cluster:
\begin{eqnarray}
	{\rm sim} = \frac{1}{3} \sum_{r=1}^{3}\;  \frac{1}{J_r^2}\; \sum_{i,j \in J_r } \; c_{ij} \le 1 
\end{eqnarray}
Thus, the similarity index will help us
to assess the quality of the clustering procedure. 

In our simulations, we consider $m_1 = m_2 = m_3 = 50$ and $|J_1|=|J_2|=|J_3| = 10$. Hence, with a fixed value of $\epsilon$, we variate the value of $\gamma$ to see the coherence of theorem \ref{thm1} with the experiment. In the synthetic data, $ \frac{1}{(m_1-l)^2} \approx 0.00062$.
However, if the chosen $\epsilon$ does not verify the hypothesis of theorem \ref{thm1}, the last selected cluster is returned. 
For given values of $\gamma$ \eqref{eq:problem}, we perform  10 times the computation (re-sampling only the  noise) then we compute the mean of the cluster quality (recovery rate) and the similarity (sim) of the output for the three modes. 
The results are shown in figure \ref{multislice_only} with the standard deviations (std) of each measure at each $\gamma$.
\begin{figure}[h]
	\begin{subfigure}[b]{0.4\textwidth}
		\centering
		\includegraphics[width=5.7cm, height=4.3cm ]{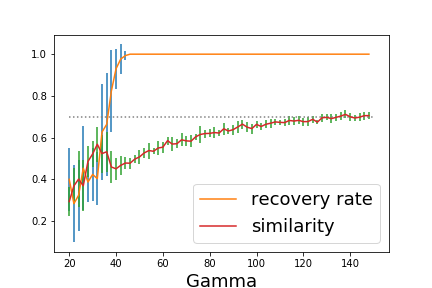}
		\caption{$\epsilon = 0.005$}
		\label{fig:rateari}
	\end{subfigure}
	\hspace{1.1em}
	\begin{subfigure}[b]{0.4\textwidth}
		\centering
		\includegraphics[width=5.7cm, height=4.3cm]{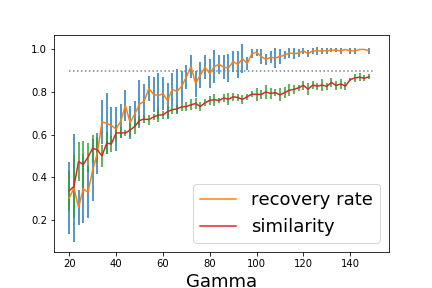}
		\caption{$\epsilon = 0.0006$}
		\label{fig:bari}
	\end{subfigure}
	\hfill
	\caption{Recovery rate of the MSC method with 
		$\gamma \in [20, 150]$: 
		(a) with $\epsilon = 0.005$ and (b) with $\epsilon = 0.0006$.} 
	\label{multislice_only}
\end{figure}
%\vspace{-0.7cm}

For a range of values of $\gamma$, 
the recovery rate of figure \ref{multislice_only} (a) increases faster than that of the figure (b). For the larger value of $\epsilon$ preventing to enter in the while loop, the algorithm reaches a high recovery rate (almost 1) very fast (with $\gamma = 50$) but note that the similarity within the cluster is low (sim$=0.5$). 
However, for the same value of $\gamma=50$, with a smaller $\epsilon$ that obeys the  hypothesis of theorem \ref{thm1}, 
some indices are removed in the cluster, see figure  \ref{multislice_only} (b), to guarantee the similarity between the selected indices. As a result, we have a smaller recovery rate equals to $0.7$ and sim$=0.65$.  Increasing the signal, 
the recovery rate
and the similarity are remarkably increased to reach high values. Thus, we have successfully tied 
the recovery rate with the similarity index.

\subsection{Comparison with existing methods}
In this section, we perform two comparisons. On one hand, 
we compare the MSC and 
CP+k-means and Tucker+k-means' methods. 
We use the same construction and data as in the previous experiment. Referring to \cite{scalableTamaraKolda}, the CP and Tucker decompositions deliver factor matrices
for each mode and then we apply  
the k-means' method to the resulting factor matrices. In our case, we use k-means with $k=2$. The ARI serves as a comparison 
tool of the quality of the results obtained by the different algorithms. We additionally compute the mean squared error (MSE) of the sub-cube cluster as a similarity test of the cluster data. 
On the other hand, we use the MSC to find a tensor biclustering and 
compare its performance with the TFS method. To do so, the MSC needs some light adjustments: the clustering is done on 2 modes
and is not performed on the third mode that will define the trajectory indices. We generate the same dataset of \cite{FeiziNIPS2017} (section 4.1 therein) with $m_3 = 50, m_1=m_2 = 70$ and $|J_1|=|J_2| = 10$ where $J_1$ and $J_2$ represent the two sets of indices in the first and second modes, respectively, that represent the correlated trajectories. We  use the recovery rate and the correlation mean of trajectories to evaluate the inference quality of the results in both methods. 
\begin{figure}[h]
	\begin{subfigure}[b]{0.4\textwidth}
		\centering
		\includegraphics[width=5.7cm, height=4.3cm ]{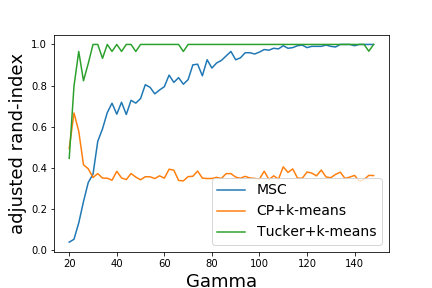}
		\caption{}
		\label{fig:msc_cp_tucker}
	\end{subfigure}
	\hspace{1.2em}
	\begin{subfigure}[b]{0.4\textwidth}
		\centering
		\includegraphics[width=5.7cm, height=4.3cm ]{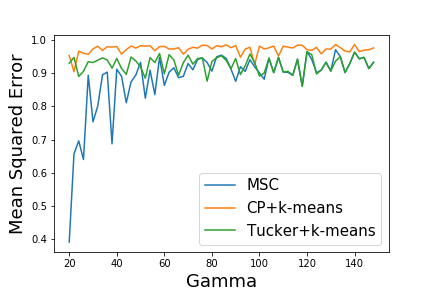}
		\caption{}
		\label{fig:tfs_msc}
	\end{subfigure}
	\caption{(a) ARI of MSC ($\epsilon=0.0006$), CP+k-means and the Tucker+k-means ($k=2$). (b) MSE of the  sub-cube cluster.
	}
	\label{fig:comparison}
\end{figure}

In figure \ref{fig:msc_cp_tucker}, 
we observe that the Tucker+k-means is the fastest method to find the best cluster but note that the ARI is not stable for $\gamma < 50$ (a lot of fluctuations that we omit in the figure). The MSC reaches an ARI $\ge 0.7$ for $\gamma = 50$ and becomes close to $0.9$ for $\gamma \ge 75$ and proves to be much stable. As we expected, the ARI of the MSC grows a bit slower than the Tucker+k-means because it is devised to select only the most similar slices in each mode. But at reasonable
signal strength, it becomes comparable 
with the Tucker+k-means. 
In any case, the CP+k-means is 
outperformed because it remains stationary even if the value of $\gamma$ increases.
Figure \ref{fig:tfs_msc} shows the MSE of the different methods.
As expected, the MSC has the best MSE for $\gamma < 75$,
whereas it fits with that of the Tucker+k-means for $\gamma \ge 80$. 
We can conclude that the MSC is certainly a significant clustering algorithm.  

\begin{figure}[h]
	\begin{subfigure}[b]{0.4\textwidth}
		\centering
		\includegraphics[width=5.7cm, height=4.3cm ]{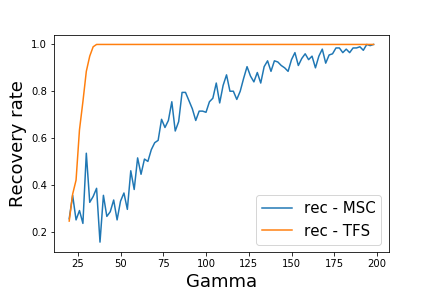}
		\caption{}
		\label{fig:msc_tfs_1}
	\end{subfigure}
	\hspace{1.2em}
	\begin{subfigure}[b]{0.4\textwidth}
		\centering
		\includegraphics[width=5.7cm, height=4.3cm ]{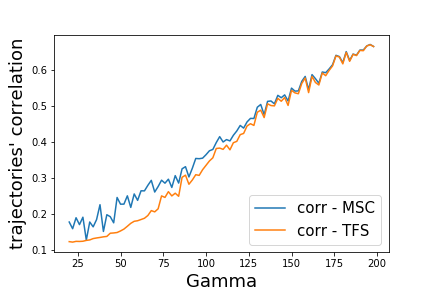}
		\caption{}
		\label{fig:tfs_msc_mse}
	\end{subfigure}
	\caption{Recovery rates (rec) and correlation mean (corr) for MSC and TFS.}
	\label{fig:comparison2}
\end{figure}

In the figure \ref{fig:comparison2}, we observe that the TFS reaches the highest recovery rate faster than the MSC method with $\epsilon = 0.00027$ 
(but remember that the TFS method takes the cardinality 
$|J_1|\times |J_2|$ as an input parameter).
The recovery rate of the MSC 
grows slowly but reaches a remarkable 
good rate at $\gamma=125$. 
We recall that the TFS method aims at selecting highly correlated trajectories. Nevertheless, 
as shown in figure \ref{fig:comparison2},
for $\gamma <120$, the trajectories which are selected 
have a very weak correlation mean and thus are not highly correlated. In  contrast, as we planed, 
the correlation mean and recovery rate
of the MSC increase together. 
This comparison shows that, combining both performances, the MSC  is certainly a more relevant method 
for clustering similar trajectories  in a $3-$order tensor compared to  the TFS.

\subsection{Real dataset}

We undertake now the study of the flow injection analysis (FIA) dataset \cite{fiadata} with size $12$ (samples) $\times 100$ (wavelengths) $\times 89$ (times). 
The MSC algorithm is tested with different values of $\epsilon\in [0.005, 00001]$. The cardinality of the cluster decreases and stabilizes for $\epsilon\in[0.00013,0.00001]$. Therefore, we choose $\epsilon = 0.00013$.
Table \ref{tab:tab1} presents the indices of similar slices in each mode and the maximum of the Frobenius norm of the difference between them ($\max_{i,j}\|slice(i) - slice(j)\|_F$ for $i,j$ in the cluster mode).
\begin{center}
	\begin{table}[h]
		\begin{tabular}{ | m{2.2cm} | m{6cm}| m{3.5cm} | } 
			\hline
			\textbf{mode} & \textbf{Indices of similar slices} & \textbf{Max Frobenius norm} \\ 
			\hline
			mode-$1$  & 10, 11 & $3.08616$\\ 
			\hline
			mode-$2$  &  39, 40, 41, 42, 43 &$ 0.40307$  \\ 
			\hline
			mode-$3$  & 45, 46, 47, 48, 49, 50, 51, 52 & $ 1.05795$ \\
			\hline
		\end{tabular}
		\vspace{0.5cm}
		\caption{The similar slices in each mode with the maximum Frobenius norm of the difference between two slices within each mode.}
		\label{tab:tab1}
	\end{table}
\end{center}
We will show that this defines a good 3D cluster with highly correlated vectors by adding a new element in each cluster mode and then experimentally demonstrate that the correlation in the cluster decreases while the MSE of the cluster increases. 

Table \ref{tab:tab3} exhibits the similarity between the fibers (trajectories) by letting one mode free and computing the correlation mean of all fibers from the two clusters of the remaining modes. 
\begin{table}[h]
	\begin{tabular}{ | m{3cm} | m{5.5cm}| } 
		\hline
		\textbf{mode} & \textbf{Fibers' correlation}  \\ 
		\hline
		mode-$1$  & 0.99915 \\ 
		\hline
		mode-$2$  & 0.97164 \\ 
		\hline
		mode-$3$  & 0.99059 \\ 
		\hline
	\end{tabular}
	\vspace{0.5cm}
	\caption{Fibers' correlation in each mode and the MSE of the triclustering.}
	\label{tab:tab3}
\end{table}
Moreover, we compute also the MSE of the triclustering (sub-cube) and obtain MSE = 0.48360.

Now, we evaluate the quality of cluster presented in table \ref{tab:tab1} and \ref{tab:tab3}. To do so, for each mode $i=1,2,3$, we add one index, randomly selected, to the cluster set $i$, 
and perform the measurements. 
This is repeated  60 times (this is the experiment E-$i$). At each time, we compute the maximum Frobenius norm of 
the slice differences, the correlation of the fibers and the MSE of the triclustering. At the end, we select the minimum for all calculations of the Frobenius norms and of the MSE, and the mean of the fibers' correlation. The results are reported in table \ref{tab:tab4}.

\begin{table}[h]
	\begin{tabular}{ | m{2cm} | m{3cm}| m{3cm}| m{3.5cm}| } 
		\hline
		\textbf{E-mod} & \textbf{Frobenius norm} & \textbf{Fibers' correlation}& \textbf{MSE of triclustering}  \\ 
		\hline
		E-1  & 3.17592 & 0.83039 &0.49257 \\ 
		\hline
		E-2  & 3.83777 & 0.86484 & 0.51133 \\ 
		\hline
		E-3  & 4.86486 & 0.84928 & 1.19725 \\ 
		\hline
	\end{tabular}
	\vspace{0.5cm}
	\caption{Results of the experiments E-1, E-2 and E-3.}
	\label{tab:tab4}
\end{table}
We observe, comparing with the above tables and
previous MSE of 0.48360,
that the Frobenius norm and MSE increase 
and the fibers' correlation decreases. 
In short, adding a new element in the output cluster decreases its quality. 
Hence, the MSC has remarkable performance
for detecting a 3D cluster, 
with strongly similar slices and 
and strongly correlated fibers in this real dataset.

\section{Conclusion and future work}
\label{ccl}

The majority of the clustering algorithms depends on parameters
which correspond to a given number of clusters or a given number of elements assigned to a cluster. Such parameters do not guarantee the most accurate clusters. 
We have introduced a new clustering method, the MSC, for $3$-order tensor dataset that tackles this issue by 
using a threshold parameter $\epsilon$ to carry out the cluster selection and refinement. 
Our main results state in 
probabilistic terms (theorems \ref{thm1})
%and \ref{thm2}, and corollary \ref{corollary}
that show the algorithm  effectiveness. The experimental validation demonstrates that our algorithm performs well. We have comparable performance (ARI) with the Tucker+k-means, 
and, adapted to biclustering, with TFS. Moreover, our algorithm guarantees strong correlation/similarity in the triclustering,
even at small signal strength, a property that is largely not verified for many  algorithms in particular TFS and Tucker+k-means. 
Concerning real data sets, the new algorithm has located a 3D cluster with highly similar slices 
and fibers. We note that MSC can be used also to detect the outlier slices in the dataset. 
A possible fruitful avenue of research would be to select a set of most dominant eigenvectors in each slice to  define a multi-clustering. 
Such a multi-clustering could be compared with 
contemporary approaches such as the multi-way clustering
for higher dimensional data. 
This deserves full-fledged investigation.

\bibliographystyle{splncs04}

\bibliography{msc_paper}

% ---------------------  Appendix  ------------------------------
\newpage

\title{Multi-Slice Clustering for $3$-order Tensor Data - Supplementary Material}

\author{Dina Faneva Andriantsiory, Joseph Ben Geloun, Mustapha Lebbah}

\institute{Laboratoire d'Informatique de Paris Nord (LIPN) \\   Université Sorbonne Paris Nord}

\maketitle

\begin{abstract}
	This document provides a 
	supplementary material
	to the manuscript 
	``Multi-Slice Clustering for $3$-order Tensor Data". 
	It gathers the proof the its main statements. 
	In particular, we prove that within a 3-cluster selected by the Multi-Slice Clustering method, the similarity of data holds with high probability.  
	
\end{abstract}

\section{Notation}

We use $\cT$ to represent the 3-order tensor data, see figure \ref{fig:tensor}.  We also use the Matlab notation in the next manipulations. For matrices, we use  capital letters: $M, C,\cdots$, and $\|M\|$ is  the operator norm of $M$. The lowercase boldface $\bx, \bv, \cdots$ represent vectors and $\|\bx\|_2$ is the euclidean norm of the vector $\bx$. For any set $J$, $|J|$ denotes the cardinality of $J$ and $\bar{J}$ denotes the complementary of $J$ in a larger set. For an integer $n>0$, we denote $[n] = \{1,\cdots,n\}$.  The asymptotic notation $a(n) = \cO(b(n))$ (res. $a(n) = \Omega(b(n))$) means that, there exists a universal constant $c$ such that for sufficiently large $n$, we have $|a(n)| \le c b(n)$ (resp. $|a(n)| \ge  c b(n)$). 

\begin{figure}[h!]
	\centering
	\begin{subfigure}[b]{0.4\textwidth}
		\centering
		\includegraphics[width=\textwidth ]{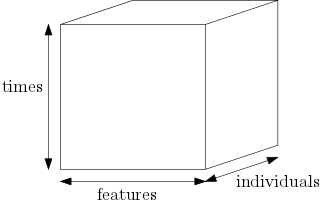}
		\caption{A 3-order tensor}
		\label{fig:tensor}
	\end{subfigure}
	\hfill
	\begin{subfigure}[b]{0.4\textwidth}
		\centering
		\includegraphics[width=\textwidth]{./image/tensor_biclustering.png}
		\caption{Tensor biclustering}
		\label{fig:tfs}
	\end{subfigure}
	\hfill
	\newline
	\begin{subfigure}[b]{0.4\textwidth}
		\centering
		\includegraphics[width=\textwidth]{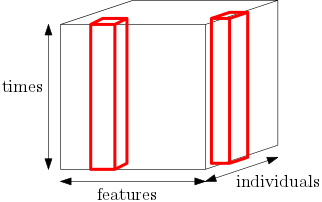}
		\caption{tensor biclustering extension}
		\label{fig:tfs_extension}
	\end{subfigure}
	\hfill
	\begin{subfigure}[b]{0.4\textwidth}
		\centering
		\includegraphics[width=\textwidth]{./image/objective.png}
		\caption{Triclustering (Our Model)}
		\label{fig:triclustering}
	\end{subfigure}
	\caption{(a) Representation of a 
		3-order tensor data set. (b) The tensor biclustering problem. (c) Multiple tensor biclustering \cite{FeiziNIPS2017}. (d) The triclustering problem.}
	\label{fig:cluster}
\end{figure}

\section{Problem formulation}\label{pbmformulation}
Let $\cT = \cX + \cZ$ with $\cT\in\R^{m_1\times m_2\times m_3}$, where $\cX$ is the signal tensor and $\cZ$ is the noise tensor. Consider

\begin{equation}
	\cT = \cX + \cZ = \sum_{i = 1}^{r}\gamma_i \bw^{(i)}\otimes \bu^{(i)}\otimes\bv^{(i)} + \cZ \label{eq:problem}
\end{equation}
where $\forall i, \gamma_i>0$ stands for the signal strength,  $\bw^{(i)}\in \R^{m_1}, \bu^{(i)}\in\R^{m_2}$ and $\bv^{(i)}\in\R^{m_3}$ and $r$ is a number of rank-1 tensor (vector) decomposition of the signal tensor. We shall restrict to the case $r = 1$.
\begin{definition}[Triclustering]
	The problem of tensor triclustering aims  at computing a triple $(J_1, J_2, J_3)$ of index sets  with $J_1\subset[m_1], J_2\subset[m_2]$ and $J_3\subset [m_3]$, where the entries of $\cT$ are similar.
\end{definition}
To simplify the notation, we drop the superscripts $(i)$ from $\bv^{(i)}$ and from the other vectors.  For every $(j_1,j_2,j_3)\in J_1\times J_2\times J_3, \bw(j_1)=\delta, \bu(j_2)=\delta$ and $\bv(j_3)=\delta$ for a constant $\delta >0$ and $\cX_{j_1,j_2,j_3}=0$ for each $(j_1,j_2,j_3)$ outside $J_1\times J_2\times J_3$. 
Concerning the noise mode, for $(j_1,j_2,j_3) \in \bar{J}_1\times \bar{J}_2\times\bar{J}_3 $, we assume that the entries of $\cZ$ are i.i.d and have a standard normal distribution. 
This is the conventional noise model often considered in unsupervised learning method for tensor data \cite{FeiziNIPS2017,Cai_2017}.

\section{Method}
\label{method}

\subsection{Multi-slice clustering}
Consider the triclustering problem \eqref{eq:problem}. 
In the notation of Kolda and Bader \cite{Kolda2009}, let 
\begin{equation}
	\label{tmatrix}
	T_i = \cT(i,:,:)=X_i + Z_i,\qquad i\in[m_1]
\end{equation}
be the horizontal (the first mode) matrix slice of the tensor $\cT$, with $X_i = \cX(i,:,:)$ and $Z_i = \cZ(i,:,:)$. Similarly, we can define the lateral (the second mode)  and the frontal (the third mode) matrix slices. A way to learn the embedded triclustering in a real tensor data set is to find the similarity of the slices in each mode. This yields an index subset  $J_i$ for each mode. Now, the collection of $J_i$ determines the tricluster.
\begin{definition}[Multi-slice clustering] 
	The problem of multi-slice clustering of a $3$-order tensor $\cT$ aims at determining and gathering the indices of the matrix slices that are similar in each dimension of the tensor. 
\end{definition}

To illustrate the procedure, we will focus on the horizontal slices. The treatment of the other slices follows the same idea. 
We start by computing and performing a spectral analysis of
\begin{equation}
	C_i = T_i^tT_i \label{covariance}
\end{equation}
for a fixed $i\in [m_1]$.
For each $i\in[m_1]$, $C_i$ represents the symmetric covariance matrix of the $i$-th horizontal slice. 

For $i\in J_1$, equation \eqref{eq:problem} shows that 
\begin{equation}
	\label{cXX}
	C_i = (X_i + Z_i )^t(X_i + Z_i ) = X_i^tX_i + X_i^tZ_i + Z_i^tX_i + Z_i^tZ_i 
\end{equation}
Taking into consideration \eqref{eq:problem}
and \eqref{tmatrix}, a tensor calculation 
shows that the matrix 
$ X_i^tX_i  = \la_i \bv_i\bv_i^t$, 
where $\la_i = \gamma_i^2 (\bw(i))^2>0$ is the top eigenvalue
of  $ X_i^tX_i$ that corresponds to 
the top eigenvector $\bv_i$. We re-express 
\begin{equation}
	C_i = \la_i \bv_i\bv_i^t + W_i  \label{covvector} 
\end{equation}
with $W_i$ the remainder that can be 
inferred from expression \eqref{cXX}.
There is a perturbation 
bound that shows the norm
difference between 
the top eigenvectors of
$C_i$ and $\bv_i$, see
section \ref{app:prooflemmaCvv}.

For $i\notin J_1$, 
the expression $C_i$ \eqref{covariance}  becomes the covariance of the noise slice:
\begin{equation}
	C_i = Z_i^t Z_i
\end{equation}
where each row of $Z_i$ is real and independently drawn form $\cN_{m_3}(0, I)$, the $m_3$-variate normal distribution with zero mean and covariance matrix $I$.

The covariance matrix $C_i$ defines both the spread (variance) and the orientation (covariance) of our
slice data. We are interested in the vector that represents the covariance matrix and its magnitude. The idea is to find the vector that points into the direction of the largest spread of the data $(\la_i, \bv_i)$.

For each slice, the  relation  \eqref{covariance} holds for $i= 1,\cdots, m_1$.
Let $\underline\la_i$ be
the top eigenvalue
and $\tbv_i$ be the 
top eigenvector of $C_i$.
We construct the matrix 
\begin{equation}
	V = \begin{bmatrix}
		\tla_1\tbv_1 & \cdots & \tla_{m_1}\tbv_{m_1}	\end{bmatrix}
\end{equation}
where  we set $\tla_i$ to $ \underline\la_i/\la, \forall i\in[m_1]$, and $\la =  \max(\underline\la_1,\cdots,\underline\la_{m_1})$.
Let $C$ be the matrix with each entry $(c_{ij})_{i,j\in[m_1]}$ defined by 
\begin{eqnarray}
	c_{ij}  = | \langle \tla_i \tbv_i,\tla_j \tbv_j \rangle| = \tla_i\tla_j| \langle  \tbv_i, \tbv_j \rangle| \le 1
	\label{scalar_product}
\end{eqnarray}
$C$ is a symmetric matrix and its  columns are not unit vectors. Note also that $C$ is related to the covariance matrix $V^tV$ in which we take the absolute values of the entries.

\begin{definition}[$\epsilon$-similarity] We call the $i^{th}$ and $j^{th}$ slices $\epsilon$-similar, if for a small $\epsilon>0$ we have,
	\begin{equation}
		c_{ij} \geq 1 - \frac{\epsilon}{2} \label{select}
	\end{equation}
\end{definition}

Let $\mathbf{d} = (d_1, \cdots, d_{m_1})$ a real vector defined by the marginal sum in $C$: $\forall i \in [m_1] $
\begin{equation}
	d_i =  \sum_{j\in [m_1]} c_{ij} \label{26}
\end{equation}
Using \eqref{select} and \eqref{26}, the elements of $J_1$ 
are some indices $i\in [m_1]$ which define  pairwise $\epsilon$-similar slices and  therefore obeys 
\begin{equation}\label{eq:d>k(1-e)}
	d_i \ge l (1-\frac{\epsilon}{2})
\end{equation}
where $l = |J_1|$.

\section{A 
	perturbation bound  
	on top eigenvectors}
\label{app:prooflemmaCvv}

First, we recall a lemma withdrawn from 
\cite{FeiziNIPS2017}:  

\begin{lemma}
	\label{lemme6}
	Let $Y=\bx \bx ^t+\sigma W$, where $\bx\in \mathbb{R}^n, \|\bx\|=1$ and $W\in \mathbb{R}^{n\times n} = \sum_{j=1}^{N}(z_jz_j^T - \mathbb{E}z_jz_j^T)$ where $z_j$ are i.i.d $\mathcal{N}(0,I_{n\times n}))$ gaussian random vectors. Let $\tilde\bx\in \R^n, \|\tilde\bx\|=1$ be the eigenvector corresponding to the largest eigenvalue of the matrix $Y$. Let the operator norm of the matrix $W$ be such that 
	\begin{equation}
		\|W\|\leq \eta_{n,N},
		\label{23}
	\end{equation}
	with probability at least $1 - \cO(n^{-2})$. Further, let
	\begin{equation}
		\sigma \leq \frac{c_0}{\eta_{n,N}},
		\label{24}
	\end{equation}
	for some positive constant $c_0 < 1/6$. Letting $M = \|\bx\|_{\infty}$, we have
	\begin{equation}
		\|\tilde\bx - \bx\|_{\infty} \leq \cO\Big(\sigma\big(\sqrt{N \log N} + M\eta_{n,N}\big)\Big)
		\label{25}
	\end{equation}
	with high probability as $n\rightarrow \infty$.
\end{lemma}

\begin{proof}
	See lemma 6 in \cite{FeiziNIPS2017}. 
\end{proof}

We will use the notation 
introduced for 
$C_i = \la_i \bv_i\bv_i^t + W_i $ \eqref{covvector}, 
given $\la_i$ the top eigenvalue and 
$\bv_i$ the top eigenvector of $X_i^tX_i$. 
In the following $n=m_3$, 
to simplify our notation. 

\begin{lemma}
	\label{lemCvv}
	Let $\tbv_i$ be the top eigenvector $C_i$. 
	There exists a bound on $\|W_i\|$,
	such that for  $\la_i = \cO(n)$, and  $\alpha = \|\bv_i\|_{\infty}$, 
	\begin{equation}
		\|\tbv_i - \bv_i\|_{\infty}\le \cO(\frac{1}{\la_i} \alpha\log(n))
	\end{equation}
	holds with high probability as $n\rightarrow \infty$.
\end{lemma}
\begin{proof} 
	Let us focus on the noise term 
	$W_i =X_i^tZ_i + Z_i^tX_i + Z_i^tZ_i$. 
	Our goal is to select a suitable bound on it in such a way to fulfill the hypothesis of lemma 
	\ref{lemme6}. 
	
	$Z_i$ obeys a standard normal distribution, 
	therefore $\|Z_i\|$ is a sub-Gaussian 
	random variable \cite{vershynin2011introduction}
	and we have 
	\begin{equation}
		\pr(\|Z_i\|>t) \leq \exp(-ct^2)
	\end{equation}
	Since 
	\begin{eqnarray}
		&&
		\|Z_i^t Z_i\| = \sup_{\mathbf{y}\in S^{m_3-1}, \bx\in S^{m_3 -1}} \langle Z_i^t Z_i\bx,\mathbf{y} \rangle =\sup_{\mathbf{y}\in S^{m_3-1}, \bx\in S^{m_3 -1}} \langle Z_i\bx,Z_i\mathbf{y} \rangle 
		\crcr
		&& =\sup_{ \bx\in S^{m_3 -1}} \langle Z_i\bx,Z_i\mathbf{x} \rangle = \|Z_i\|^2, 
	\end{eqnarray}
	we have $\pr(\|Z_i^t Z_i\|>t)  \leq  \exp(-ct)$.
	Fix $t = c'\log(n)$, with $c'>2/c,$ to reach
	\begin{equation}
		\pr\big(\|Z_i^t Z_i\|>c'\log(n)\big) \leq  n^{-cc'}. \label{ZZ}
	\end{equation}
	To have each entry of $Z_i^tZ_i$ centered, we take $Z_i^tZ_i - nI$ where $I$ is the identity matrix of order $n$. Consequently, the mean of the noise is a scaled identity matrix, substracting this term does not change the eigenvector structure.
	
	Next, we bound the operator norm of $X_i^tZ_i$ in probability. Each entries of $X^t_iZ_i$ is a centered sub-gaussian random variable.  We use  
	again the same recipe. For a constant $c>0$ and $t =c'\sqrt{\log(n)}$, $c'>2/c,$ we obtain
	\begin{equation}
		\pr(\|X^tZ\|> c'\sqrt{\log(n)}) \leq n^{-cc'}
		\label{XZ}
	\end{equation}
	Combining \eqref{ZZ} and \eqref{XZ}, 
	we get an upper bound of the noise
	term as 
	\begin{eqnarray}
		\|W_i\|\le c'(\sqrt{\log(n)}+\log(n)) = \eta_{n,N}
	\end{eqnarray}
	that holds with probability 
	$1 - \cO(n^{-cc'}) \ge 1 - \cO(n^{-2})$,
	as $n$ goes to infinity. 
	
	We re-write $C_i /\la_i  = \bv_i\bv_i^t + W_i/\la_i$, 
	and set $\sigma = 1/\la_i$. 
	We verify that 
	\begin{equation}
		\sigma\le \frac{c_0}{
			c' (\sqrt{\log(n)}+\log(n))}
	\end{equation}
	which certainly occurs for large
	enough $n$. Then apply the
	$l_{\infty}$-perturbation bound
	of lemma \ref{lemme6} to deduce
	that for $\hat\bv_i$ the top  $C_i /\la_i$ that 
	\begin{equation}
		\|\hat\bv_i  - \bv_i\|_{\infty}
		\leq \cO\Big(\sigma \big(0 + M\eta_{n,N} \big)\Big)
		\le  
		\cO(\frac{1}{\la_i} \alpha\log(n))
	\end{equation}
	after fixing $N=1$, $\eta_{n,1}=c' (\sqrt{\log(n)}+\log(n))$, $M = \alpha = \|\bv_i\|_\infty$. 
	The same holds for $\tbv_i$ top
	eigenvector of $C_i$. 
	
	\qed 
\end{proof}

\section{Proofs of theorems}
\label{app:proofofthm}

\subsection{Fundamental lemma on 
	TW distribution}
\label{app:lem3}

For $i\notin J_1$, 
i.e. out of the cluster, 
$C_i$ \eqref{covariance}  becomes the covariance of the pure noise slice:
\begin{equation}
	C_i = Z_i^t Z_i
\end{equation}
where each row of $Z_i$ is a real matrix independently drawn form $\cN_{m_3}(0, I)$.  $C_i$ has a white Wishart distribution $\cW_{m_3}(m_2, I)$ \cite{johnstone2001}.
It is also known that, 
at large $m_2$ and large $m_3$, such that $m_2/m_3 \to \gamma >1$, 
the largest eigenvalue distribution $\la_i$ of $C_i$  fulfills at the limit the following:  
\begin{equation}\label{convergenceDistribution2}
	\frac{\la_i - \mu_{m_2m_3,1}}{\sigma_{m_2m_3,1}}\xrightarrow[]{\cD} F_1
\end{equation}
where $F_1$ is the Tracy-Widom (TW) cumulative distribution function (cdf), with parameters
\begin{eqnarray}
	\mu_{m_2m_3,1} &=& (\sqrt{m_2-1} + \sqrt{m_3})^2\crcr
	\sigma_{m_2m_3,1} &=& \sqrt{\mu_{m_2m_3,1}} \big(\frac{1}{\sqrt{m_2 - 1}}+\frac{1}{\sqrt{m_3}}\big)^{\frac{1}{3}}
\end{eqnarray}
In \cite{johnstone2001}, there could be  a limit such that  \eqref{convergenceDistribution2} still holds  if $m_2<m_3$, 
both $m_2$ and $m_3$ still going to infinity, and if the role of $m_2$ and $m_3$ are reversed in the expression of $\mu_{m_2m_3,1}$ and $\sigma_{m_2m_3,1}$. It is also claimed that  \eqref{convergenceDistribution2} is valid at the limit, with a  satisfactory approximation whenever $m_2, m_3 \ge 10$. 
We will be above that range for these parameters in our next  computations. 
We will drop the subscript $(m_2m_3,1)$, and simply write  $\mu = \mu_{m_2m_3,1}$ and $\sigma = \sigma_{m_2m_3,1}$. 

The gamma distribution approximates the TW and is given by the following probability density function \cite{ietVlokOlivier2012}
\begin{equation}\label{pdfTWApprox}
	p(x) = \frac{(x-x_0)^{k-1}}{\theta^k\Gamma(k)}\exp\Big[\frac{-(x-x_0)}{\theta}\Big]
\end{equation}
where $x_0 = -9.8209$ is the location (shift) parameter, $k=46.5651$ 
calls the shape, $\theta=0.1850$ the scale and $\Gamma(k)$ the Gamma function. This mean 
of this random variable equals $k\theta + x_0$.

The following lemma is crucial 
in the proof of theorem \ref{thm1}. 

\begin{lemma}\label{lem1}
	Let $Z_i$ be a $M\times N$ matrix for $i\in [n]$ where for each $i$, each row of $Z_i$ is real and independently $N$-variate normal distribution with zero mean and covariance matrix $I$. Then the $N\times N$ matrix $Y_i=Z_i^tZ_i$ has a white Wishart distribution $\cW_N(M, I)$. Denote $\la_1, \cdots,\la_n$ be the largest eigenvalue of $Y_i$ for $i\in[n]$. For $\mathbf{a}=(a_1,\cdots,a_n)$ such that $\forall i, a_i\ge 0$,
	\begin{itemize}
		\item if $\sum_i^n a_i \le 1$, for $\la = \Omega(\mu)$, we have
		\begin{equation}
			\frac{1}{\la}	\sum_{i=1}^{n} a_i\la_i \le \sqrt{\log(n)}
		\end{equation}
		with probability at least $1 - en^{-c_1}$ for a constant $c_1>0$.

		\item if $1 <\sum_i^n a_i \le n $, for $\la = \Omega(n\mu)$, we have
		\begin{equation}
			\label{Onmu}
			\frac{1}{\la}	\sum_{i=1}^{n} a_i\la_i \le \sqrt{n}
		\end{equation}
		with probability at least $1-e n^{-c_2}$ for a constant $c_2>0$.
	\end{itemize}
\end{lemma}
\begin{proof} 
	By definition of Wishart distribution with the gamma approximation of TW random variable, we know that 	
	$	\frac{\la_i -\mu}{\sigma} \xrightarrow[]\cD F_1.$
	The expected value of the TW  for real random variable is equal to $k\theta + x_0$. For $t>0$ we have
	\begin{eqnarray}
		&& 
		\pr\Big(|\sum_i a_i (\frac{\la_i-\mu}{\sigma}  -k\theta-x_0)|\ge t\Big)
		%&=& \pr\Big(\sum_i a_i(\frac{\la_i-\mu}{\sigma}-k\theta-x_0)\ge t\Big) \crcr
		%	&& + \pr\Big(\sum_i a_i (\frac{\la_i-\mu}{\sigma}  -k\theta  - x_0)\le - t\Big)\crcr
		\ge \pr\Big(\sum_i a_i (\frac{\la_i-\mu}{\sigma}  -k\theta - x_0)\ge t\Big) \crcr
		&& \ge  \pr\Big(\sum_i a_i \la_i \ge t\sigma +  \Big[\sigma(k \theta + x_0) + \mu\Big]\sum_i a_i\Big)\crcr
		&& \ge \pr\Big(\frac{1}{\la}\sum_i a_i \la_i \ge \frac{t\sigma}{\la} +  \frac{\Big[\sigma(k \theta + x_0) + \mu\Big]}{\la}\sum_i a_i\Big)\crcr
		&& \ge\pr\Big(\frac{1}{\la}\sum_i a_i \la_i \ge\frac{t\sigma}{\la} +  \frac{(\sigma k\theta + \mu)}{\la}\sum_i a_i\Big)
	\end{eqnarray}
	For $\sum_i a_i \le 1$, we have
	\begin{equation}
		\pr\Big(|\sum_i a_i (\frac{\la_i-\mu}{\sigma}  -k\theta-x_0)|\ge t\Big)\ge \pr\Big(\frac{1}{\la}\sum_i a_i \la_i \ge\frac{t\sigma}{\la} +  \frac{(\sigma k\theta + \mu)}{\la}\Big)
	\end{equation}
	By hypothesis, $\la =\Omega(\mu)$, and take $t = \sqrt{\mu\log(n)} $, with
	\begin{equation}
		\frac{\sigma k \theta}{c\mu} =\frac{k\theta}{c\sqrt{\mu}}(\frac{1}{\sqrt{M-1}} + \frac{1}{\sqrt{N}})^{\frac{1}{3}} 
	\end{equation}
	\begin{equation}
		\frac{t\sigma}{c \mu } 
		%\frac{\sigma\sqrt{\mu\log(n)}}{c\mu}
		=(\frac{1}{\sqrt{M-1}} + \frac{1}{\sqrt{N}})^{\frac{1}{3}}\frac{\sqrt{\log(n)}}{c}
	\end{equation}
	For $c > 2$, we have $\frac{t\sigma}{\la} +  \frac{(\sigma k\theta + \mu)}{\la} \le \sqrt{\log(n)}$, then
	\begin{equation} \label{34}
		\pr\Big(\frac{1}{\la}\sum_i a_i \la_i > \frac{t\sigma}{\la} +  \frac{(\sigma k\theta + \mu)}{\la} \Big) \ge \pr\Big(\frac{1}{\la}\sum_i a_i \la_i \ge \sqrt{\log (n)}\Big)
	\end{equation}
	Using proposition 5.10 \cite{vershynin2011introduction}
	(Hoeffding-type inequality),
	with $\sigma_m =\max_i \|\frac{\la_i - \mu}{\sigma}\|_{\psi_2}$ and for $c>0$ an absolute constant,
	\begin{eqnarray}
		\pr\Big(\frac{1}{\la}\sum_i a_i \la_i \ge \sqrt{\log (n)}\Big)
		%&\le& e \exp(\frac{-ct^2}{\sigma_m\|a\|_2^2})\crcr
		\le e \exp(\frac{-c\mu\log(n)}{\sigma_m\|a\|_1^2}) \le  e \exp(\frac{-c\mu\log(n)}{\sigma_m})
	\end{eqnarray}
	%Then \begin{equation}
	%		\pr\Big(\frac{1}{\la}\sum_i %a_i \la_i \ge \sqrt{\log (n)}\Big) %\le en^{-c_1}
	%\end{equation}
	We introduce $c_1 = c\frac{\mu}{\sigma_m}>0$, 
	and therefore 
	$$\frac{1}{\la}\sum_i a_i \la_i \le \sqrt{\log (n)}$$
	holds  with probability at least $1 - en^{-c_1}$.
	%for a constant $c_1 > 0$.
	
	We now address the second
	statement \eqref{Onmu}.
	For $1 <\sum_i a_i \le n$, we have
	\begin{equation}
		\pr\Big(\frac{1}{\la}\sum_i a_i \la_i >\frac{t\sigma}{\la} +  \frac{(\sigma k\theta + \mu)}{\la}\sum_i a_i\Big) \ge \pr\Big(\frac{1}{\la}\sum_i a_i \la_i >\frac{t\sigma}{\la} +  \frac{(\sigma k\theta + \mu)}{\la}n\Big)
	\end{equation}
	By hypothesis, $\la =\Omega(\mu n)$, and take $t = \mu n^{\frac{3}{2}}$.
	For $c >2$, we have $ \frac{(\sigma k\theta + \mu)}{\la} \le \sqrt{n}$, then
	\begin{equation}
		\pr\Big(\frac{1}{\la}\sum_i a_i \la_i > \frac{t\sigma}{\la} +  \frac{(\sigma k\theta + \mu)}{\la}n\Big) \ge \pr\Big(\frac{1}{\la}\sum_i a_i \la_i > \sqrt{n}\Big).
	\end{equation}
	We make use, once again, of proposition 5.10 in \cite{vershynin2011introduction}, 
	with the same above $\sigma_m$, and for $c>0$ an absolute constant,
	\begin{eqnarray}
		\pr\Big(\frac{1}{\la}\sum_i a_i \la_i \ge \sqrt{n}\Big)
		%&\le& e \exp(\frac{-ct^2}{\sigma_m\|a\|_2^2})\crcr
		&\le& e \exp(\frac{-c\mu^2 n^{3}}{\sigma_m\|a\|_1^2}) 
		%\le  e \exp(\frac{-c\mu^2 n^{3}}{n^2\sigma_m})\crcr
		\le  e \exp(\frac{-c\mu^2 n}{\sigma_m})
	\end{eqnarray} 
	Then, setting 
	%\begin{equation}
	%	\pr\Big(\frac{1}{\la}\sum_i a_i \la_i \ge %\sqrt{n}\Big) \le e\exp(-c_2 n)
	%\end{equation}
	%where 
	$c_2 = c\frac{\mu^2}{\sigma_m}>0$, 
	we obtain that 
	$$\frac{1}{\la}\sum_i a_i \la_i \le \sqrt{n}$$ holds with probability at least $1 - e\exp(-c_2 n)$. 
	
	\qed\end{proof}

We are in position to 
prove our main statement.

\subsection{Proof of theorem \ref{thm1}}
\label{app:proofthm1}
\begin{theorem}\label{thm1}	Let $l = |J_1|$, assume that $\sqrt{\epsilon}\le \frac{1}{m_1-l}$. $\forall i,n\in J_1 $, for $\la = \Omega(\mu)$, there is a constant $c_1>0$ such that
	\begin{equation}
		\label{didn}
		|d_i - d_{n}| \le l\frac{\epsilon}{2} + \sqrt{\log(m_1-l)}
	\end{equation}
	holds with probability at least $1-e(m_1-l)^{-c_1}$.
\end{theorem}

\begin{proof} 
	We start by the expression \eqref{26}. 
	Let $i,n \in J_1$, 
	\begin{equation}
		d_n -d_{i}= \sum_{j\in J_1} (c_{nj} - c_{ij}) + \sum_{j\in \bar{J}_1}( c_{nj} - c_{ij})
		\label{dn-di}
	\end{equation} 
	for $j\in J_1$, we have $1>c_{nj},c_{ij} > 1-\frac{\epsilon}{2}$
	\begin{eqnarray}
		1 - (1-\frac{\epsilon}{2}) =\frac{\epsilon}{2} \ge c_{ij}- (1-\frac{\epsilon}{2}) \ge c_{ij} - c_{nj}
	\end{eqnarray}
	Using the same approach, we have $\frac{\epsilon}{2} \ge |c_{ij} - c_{nj}|, \forall i,n\in J_1$. We evaluate
	\begin{equation}
		|\sum_{j\in J_1} (c_{nj} - c_{ij})|\le l \frac{\epsilon}{2} 
	\end{equation}
	Concerning the second term in \eqref{dn-di}, 
	we write 
	\begin{eqnarray}
		|\sum_{j\notin J_1}( c_{nj} - c_{ij})| 
		%&=&  |\sum_{j\notin J_1}(\tla_n\tla_j|\langle \tbv_n, \tbv_j\rangle| - \tla_i\tla_j|\langle \tbv_i, \tbv_j\rangle|)| \crcr
		&=&|\sum_{j\notin J_1}(\tla_n|\langle \tbv_n, \tbv_j\rangle| - \tla_i|\langle \tbv_i, \tbv_j\rangle|)\tla_j|\crcr
		%	&=&|\sum_{j\notin J_1}(|\langle\tla_n \tbv_n, \tbv_j\rangle| - |\langle\tla_i \tbv_i, \tbv_j\rangle|)\tla_j|\crcr
		&\le&|\sum_{j\notin J_1}(|\langle\tla_n \tbv_n - \tla_i \tbv_i, \tbv_j\rangle|)\tla_j|\crcr
		&\le&\sum_{j\notin J_1}\|\tla_n \tbv_n - \tla_i \tbv_i\|\| \tbv_j\|\tla_j\crcr
		&=&\sum_{j\notin J_1}\|\tla_n \tbv_n - \tla_i \tbv_i\|\tla_j
	\end{eqnarray}
	By definition of indices $n$ and $i$, \eqref{noise} and  \eqref{select}, the following
	bound is found
	\begin{eqnarray}
		\|\tla_n \tbv_n - \tla_i \tbv_i\|^2 = \tla_n^2+\tla_i^2 - 2\langle \tbv_n,\tbv_i\rangle \le \epsilon
	\end{eqnarray}
	Thus, for $\sqrt{\epsilon} \le \frac{1}{m_1-l}, \sum_{j\notin J_1}\|\tla_n \tbv_n - \tla_i \tbv_i\|\le 1$, using lemma \ref{lem1} we have
	\begin{equation}
		|\sum_{j\notin J_1}( c_{nj}' - c_{ij}')| \le \sqrt{\log(m_1-l)}
	\end{equation}
	with probability at least $1 - e(m_1-l)^{-c_1}$ for a constant $c_1>0$.
	
	\qed
\end{proof}
\subsection{Proof of theorem \ref{thm2}}
\label{app:proofthm2}

\begin{theorem}\label{thm2}
	For $i\in\bar{J}_1$, if $\la = \Omega(\mu m_1)$,
	\begin{equation}
		d_i\le \frac{l}{m_1} + \sqrt{\log(m_1-l)}
	\end{equation} 
	with probability at least $1-e(m_1-l)^{-c_1}$ with $c_1>0$. 
\end{theorem}

\begin{proof} 
	
	Starting by \eqref{26}, we 
	recast $d_i$ in the form 
	$d_i = \Big(\sum_{r\in J_1} +\sum_{r\in\bar{J}_1}\Big)  \tla_r\tla_i|\langle\tbv_r, \tbv_i\rangle|
	$. Then, we bound 
	$$\sum_{r\in J_1} \tla_r\tla_i|\langle\tbv_r, \tbv_i\rangle| = \tla_i\sum_{r\in J_1} \tla_r|\langle\tbv_r, \tbv_i\rangle|\le\tla_i\sum_{r\in J_1} \tla_r\|\tbv_r\| \|\tbv_i\| \le l \tla_i$$
	The second step uses \cite{ietVlokOlivier2012}, 
	that states that the support of the pdf of $F_1$ is within $[x_0,2k\theta + x_0]$. Therefore, we have 
	%\begin{equation}
	%	\frac{\la_i-\mu}{\sigma} \le 2k\theta %+x_0 \implies 
	%	\la_i\le\sigma(2k\theta +x_0) + \mu
	%\end{equation}
	%and
	\begin{equation}
		\frac{ \la_i}{\la}\le\frac{\sigma(2k\theta +x_0) + \mu}{\la}. \label{1/m_1}
	\end{equation}
	Since $\la = \Omega(\mu m_1)$, then $\frac{\sigma(2k\theta +x_0) + \mu}{\la} \le \frac{1}{m_1}$. 
	Hence, 
	\begin{equation}
		\sum_{r\in J_1} \tla_r\tla_i|\langle\tbv_r, \tbv_i\rangle| \le \frac{l}{m_1}\,, 
		\qquad\quad 
		\sum_{r\in\bar{J}_1} \tla_r\tla_i|\langle\tbv_r, \tbv_i\rangle| \le  
		\tla_i\sum_{r\in \bar{J}_1} \tla_r 
	\end{equation}
	%$$\sum_{r\in\bar{J}_1} \tla_r\tla_i|\langle\tbv_r, \tbv_i\rangle| 
	%= \tla_i\sum_{r\in \bar{J}_1} \tla_r|\langle\tbv_r, \tbv_i\rangle|
	%\le
	%\tla_i\sum_{r\in \bar{J}_1} \tla_r\|\tbv_r\| \|\tbv_i\| =
	%\tla_i\sum_{r\in \bar{J}_1} \tla_r 
	%$$
	The previous inequality \eqref{1/m_1}
	leads us to 
	\begin{equation}
		\tla_i\sum_{r\in \bar{J}_1} \tla_r \le \frac{1}{m_1}\sum_{r\in \bar{J}_1} \tla_r\quad\text{ and }\quad \frac{|\bar{J}_1|}{m_1}\le 1
	\end{equation}
	Using lemma \ref{lem1} with $\la = \Omega(\mu m_1)$, we come to 
	\begin{equation}
		\sum_{r\in\bar{J}_1} \tla_r\tla_i|\langle\tbv_r, \tbv_i\rangle| \le\sqrt{\log(m_1 -l)}
	\end{equation}
	with probability at least $1- e(m_1-l)^{-c_1}$ for a constant $c_1>0$.
	
	\qed \end{proof}

\section{Sufficient conditions for rank one matrix similarity}
\label{proofCiCj}

We discuss in this appendix a
sufficient condition that
minimizes the norm distance 
between the slices $C_i$. 
We use the notation 
$\tla_i$ for the 
normalized top eigenvalue 
and $\tbv_i$ the top eigenvector of the matrix slice $C_i$.

\begin{proposition}\label{CiCj}
	Let two slices be coined by $i$ and $j$. 
	For $\tla_i = \tla_j$, 
	and $\langle\tbv_i, \tbv_j\rangle = 1$,	$\|C_i - C_j\| $ is minimal. 
\end{proposition}

\begin{proof} 
	
	Let $\bx = \tla_i\tbv_i-\tla_j\tbv_j$.
	%is the previous value of $\la_i$ devided by the maximum. 
	\begin{figure}[ht]
		\centering
		\begin{subfigure}[b]{0.3\textwidth}
			\centering
			\includegraphics[width=\textwidth ]{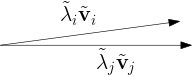}
			\caption{}
			\label{fig:a}
		\end{subfigure}
		\hfill
		\begin{subfigure}[b]{0.3\textwidth}
			\centering
			\includegraphics[width=\textwidth]{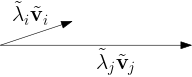}
			\caption{}
			\label{fig:b}
		\end{subfigure}
		\hfill
		\begin{subfigure}[b]{0.3\textwidth}
			\centering
			\includegraphics[width=\textwidth]{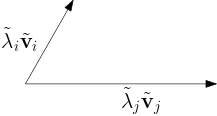}
			\caption{}
			\label{fig:c}
		\end{subfigure}
		\caption{Representation of the relation between the top eigenvectors and eigenvalues of the slice $i$ and the slice $j$.}
		\label{fig:three graphs}
	\end{figure}
	For $i,j \in [m_1]$,
	\begin{eqnarray}
		\|\tla_{i} \tbv_{i}\tbv_{i}^t - \tla_{j} \tbv_{j}\tbv_{j}^t\| &=& \|\tla_{i} \tbv_{i}\tbv_{i}^t - (\tla_{i} \tbv_{i} + \bx)(\tla_i\tbv_{i} + \bx)^t/\tla_j\|\label{signal} \\
		%	&=& \|\tla_{i} \tbv_{i}\tbv_{i}^t - (\tla_{i}^2 \tbv_{i}\tbv_{i}^t +\tla_i \bx\tbv_{i}^t +\tla_i\tbv_{i}\bx^t + \bx\bx^t)/\tla_j\|\crcr
		&=& \| (\tla_{i}- \frac{\tla_{i}^2}{\tla_j})\tbv_{i}\bv_{i}^t +\frac{\tla_i}{\tla_j}(\bx\tbv_{i}^t +\tbv_{i}\bx^t) + \frac{1}{\tla_j}\bx\bx^t\|\crcr
		&\leq&\| (\tla_{i}- \frac{\tla_{i}^2}{\tla_j})\tbv_{i}\tbv_{i}^t\| +\|\frac{\tla_i}{\tla_j}(\bx\tbv_{i}^t +\tbv_{i}\bx^t)\| +\| \frac{1}{\tla_j}\bx\bx^t\|
		\nonumber
	\end{eqnarray}
	For the first term, we have
	\begin{equation}
		\| (\tla_{i}- \frac{\tla_{i}^2}{\tla_j})\tbv_{i}\tbv_{i}^t\|  = |\tla_{i}- \frac{\tla_{i}^2}{\tla_j}|\| \tbv_{i}\tbv_{i}^t\| = |\tla_{i}- \frac{\tla_{i}^2}{\tla_j}|  \label{first}
	\end{equation}
	%Hence, the first expression gives us the necessary condition about the eigenvalues.
	We deal with the second term:
	\begin{eqnarray}
		\|\bx\tbv_{i}^t\| = \sup_{\bu\in S^{m_3-1}} \|\bx\tbv_{i}^t\bu\|_2 = \sup_{\bu\in S^{m_3-1}} |\langle \tbv_{i},\bu\rangle|\|\bx\|_2 = \|\bx\|_2
		%\crcr
		%	\|\tbv_{i}\bx^t\| &=&\sup_{\bu\in S^{m_3-1}} \|\tbv_{i}\bx^t\bu\| = \sup_{\bu\in S^{m_3-1}} |\langle \bx,\bu\rangle|\|\tbv_i\| = \|\bx\|. \label{semi_noise}
	\end{eqnarray}
	Likewise $	\|\tbv_{i}\bx^t\| = \|\bx\|_2$. 
	The third expression can be
	expressed as
	\begin{equation}\| \bx\bx^t\| = \sup_{\bu\in S^{m_3-1}} \|\bx\bx^t\bu\|_2  \label{noise_op} \end{equation}
	The upper bound of \eqref{noise_op} is reached at $\bu = \bx/\|\bx\|^2_2$. Therefore,
	\begin{eqnarray}
		\| \bx\bx^t\| &=& \| \bx\bx^t\frac{\bx}{\|\bx\|_2}\|_2 %= \| \bx\frac{\|\bx\|^2_2}{\|\bx\|_2}\|_2 
		=\|\bx\|_2^2\crcr
		%	&=& %\|\tla_i \tbv_i - \tla_j \tbv_j\|^2_2 
		%	 \| \tla_i \tbv_i\|^2_2+\| \tla_j \tbv_j\|^2_2 - 2\langle\tla_i \tbv_i , \tla_j \tbv_j\rangle\crcr
		&=& \tla_i^2 + \tla_j^2 - 2\tla_i\tla_j \langle\tbv_i, \tbv_j\rangle \label{noise}
	\end{eqnarray}
	Thus to minimize \eqref{signal}, we inspect the following expressions: 
	\begin{equation}
		\left \{
		\begin{array}{rcl}
			&|\tla_{i}- \frac{\tla_{i}^2}{\tla_j}| \\
			&\tla_i^2 + \tla_j^2 - 2\tla_i\tla_j \langle\tbv_i, \tbv_j\rangle.
		\end{array}
		\right.
		\label{equiv_signal} 
	\end{equation}
	%	 	 To find the slice clustering, we minimize the equation \eqref{signal}. So we set the two expression in equation \eqref{equiv_signal} close to zero and find all the required condition to reach the minimum.
	To reach the minimum, 
	we set 	 %	\begin{itemize}
	%\item[*] %The first expression is minimum when 
	$|\tla_{i}- \frac{\tla_{i}^2}{\tla_j}| = 0 $ and that entails $\tla_i = \tla_j$. 
	%	So, the first condition is to have $\tla_i = \tla_j$
	%	\item[*] 
	%We use the first condition to find the minimum of the second expression,
	The second term is again set to 
	0 and we collect 	 		\begin{eqnarray}\tla_i^2 + \tla_j^2 - 2\tla_i\tla_j \langle\tbv_i, \tbv_j\rangle= 0 \Rightarrow 2\tla_i(1-\langle\tbv_i, \tbv_j\rangle) = 0
		%\crcr
		%	 	&\Rightarrow&	\langle\tbv_i, \tbv_j\rangle = 1
	\end{eqnarray}
	Therefore, the second condition for 
	having a minimal norm
	between slices 
	amounts to $\langle\tbv_i, \tbv_j\rangle = 1$.
	% 	\end{itemize}

	\qed\end{proof}

\end{document}